\documentclass{article}
\usepackage{ijcai24}

\usepackage{times}
\usepackage{soul}
\usepackage{url}
\usepackage[hidelinks]{hyperref}
\usepackage[utf8]{inputenc}
\usepackage[small]{caption}
\usepackage{graphicx}
\usepackage{amsmath}
\usepackage{amsthm}
\usepackage{booktabs}
\usepackage{algorithm}
\usepackage{algorithmic}
\usepackage[switch]{lineno}
\usepackage{subcaption}
\usepackage{amsfonts} 
\usepackage{comment}

\newtheorem{theorem}{Theorem}

\title{Privacy-Preserving UCB Decision Process Verification via zk-SNARKs}

\author{
Xikun Jiang$^{1,3}$
\and
He Lyu$^3$\and
Chenhao Ying$^{1,2}$\footnote{Corresponding author}\and
Yibin Xu$^3$\and
Boris D{\"u}dder$^3$\and
Yuan Luo$^{1,2}$\\
\affiliations
$^1$Department of Computer Science, Shanghai Jiao Tong University, China\\
$^2$Shanghai Jiao Tong University (Wuxi) Blockchain Advanced Research Center, China\\
$^3$Department of Computer Science, University of Copenhagen, Denmark\\
\emails
\{xikunjiang, yingchenhao,yuanluo\}@sjtu.edu.cn,
lyuhe2000@gmail.com,
\{xikun, yx, boris.d\}@di.ku.dk}

\begin{document}

\maketitle

\begin{abstract}
With the increasingly widespread application of machine learning, how to strike a balance between protecting the privacy of data and algorithm parameters and ensuring the verifiability of machine learning has always been a challenge. This study explores the intersection of reinforcement learning and data privacy, specifically addressing the Multi-Armed Bandit (MAB) problem with the Upper Confidence Bound (UCB) algorithm. We introduce zkUCB, an innovative algorithm that employs the Zero-Knowledge Succinct Non-Interactive Argument of Knowledge (zk-SNARKs) to enhance UCB. zkUCB is carefully designed to safeguard the confidentiality of training data and algorithmic parameters, ensuring transparent UCB decision-making.

Experiments highlight zkUCB's superior performance, attributing its enhanced reward to judicious quantization bit usage that reduces information entropy in the decision-making process. zkUCB's proof size and verification time scale linearly with the execution steps of zkUCB. This showcases zkUCB's adept balance between data security and operational efficiency. This approach contributes significantly to the ongoing discourse on reinforcing data privacy in complex decision-making processes, offering a promising solution for privacy-sensitive applications.

\end{abstract}

\section{Introduction}

The integration of reinforcement learning (RL) algorithms in various fields, such as healthcare, autonomous driving, and recommendation systems, signifies a major advancement in artificial intelligence~\cite{yu2021reinforcement,afsar2022reinforcement,kiran2021deep}. The Multi-Armed Bandit (MAB) problem, a core RL model, is particularly notable for its efficient decision capabilities. The MAB model has found diverse applications, from optimizing healthcare treatments to enhancing digital user experiences~\cite{patil2021achieving,ameko2020offline,gao2022combination}.

However, the deployment of the MAB model in sensitive areas like healthcare involves significant privacy and security challenges~\cite{zhao2020differentially,azize2022privacy}. For instance, in the healthcare domain where AI systems using MAB models, such as AI doctors (AI giving diagnosis to the patients), are tasked with critical responsibilities like adjusting antidepressant dosages~\cite{aziz2021multi}. However, this application involves many issues. (1) Data Privacy Risks: It necessitates the long-term accumulation of extensive patient data, encompassing sensitive information like physical conditions and lifestyle habits, posing the risk of privacy breaches; (2) Confidentiality of Algorithm Parameters: The algorithm's parameters, refined through comprehensive experiments and continual enhancements, hold significant value for medical institutions. Therefore, ensuring the confidentiality of these parameters is imperative; (3) Verification by Third Parties: With the involvement of numerous third-party entities such as health insurance companies and government welfare agencies, there is a need to verify the application's execution accuracy. This verification must be conducted without compromising any private data or revealing sensitive algorithm parameters.

\begin{figure}
    \centering
    \includegraphics[width=0.45\textwidth]{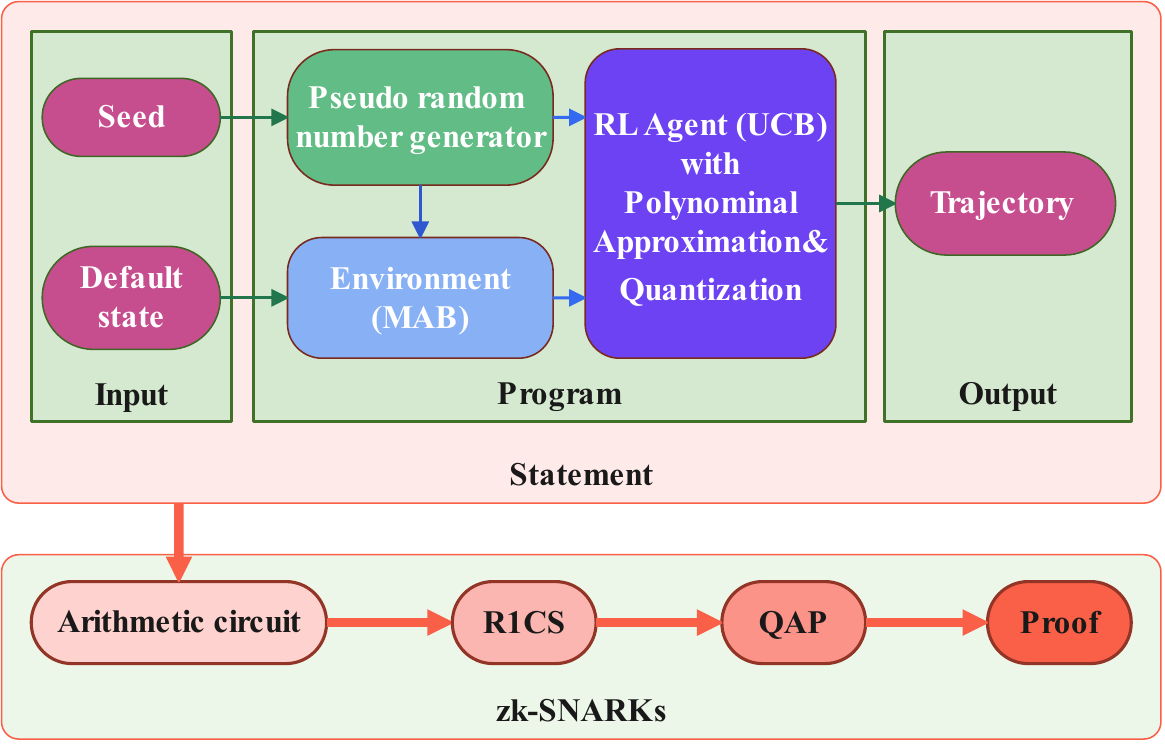}
    \caption{An illustration of zkUCB. We integrate zk-SNARK with the UBC algorithm. In zkUCB, the complete reinforcement learning process is encapsulated as a statement, which is used to model a deterministic arithmetic circuit that uses addition gates and multiplication gates for arithmetic calculations. The arithmetic circuit then serves as the first step in zk-SNARK execution.
    }
    \label{zkUCB}
\end{figure}

In this context, our study introduces an innovative solution by integrating the zero-knowledge succinct non-interactive arguments of knowledge (zk-SNARKs)~\cite{groth2016size,wahby2018doubly} with the UCB algorithm in the MAB model. This approach enables AI doctors to utilize algorithm parameters and real-time patient data for effective medical advice and medication planning, while simultaneously ensuring that no additional sensitive information is disclosed. To better understand the specific workflow of this application, we give a detailed example of integrating the UCB algorithm and zk-SNARKs and applying it to the AI doctor system. Please refer to Section~\ref{zkUCB_usecase} and Figure.~\ref{AI doctor}.

Generally speaking, applying zk-SNARKs involves converting the corresponding learning algorithm or program into a format compatible with zero-knowledge proofs. The conversion process involves first converting the algorithm or program into a deterministic arithmetic circuit and then converting that arithmetic circuit into a Rank-1 constrained system (R1CS), which is essentially a system of linear equations. These R1CS are then converted into Quadratic Arithmetic Programs (QAPs), thereby converting the linear equations into polynomial forms. QAP can then be used to generate the corresponding zk-SNARK proof, which is used to verify the output of the algorithm while protecting the confidentiality of the underlying data and the internal mechanisms of the algorithm.

\paragraph{Research Challenges.} 
Thus, to integrate the UCB algorithm with zk-SNARKs to ensure the data and parameters' confidentiality while supporting secure UCB decision process verification, we need to follow the above conversion process to convert the UCB algorithm into a format compatible with zero-knowledge proofs, thus facilitating efficient verification. However, this integration presents three major challenges: (1) The inherent randomness of the UCB complicates this task; the challenge is how to transform this randomness into the deterministic process necessary for constructing the arithmetic circuit; (2) Given that zk-SNARK relies on polynomials in finite fields Operations, how to convert the non-polynomial operations of the UCB (such as logarithms and non-integer powers) into polynomial operations is the second challenge. (3) To meet the finite field requirements of R1CS, how to adjust the floating point numbers in the UCB to meet the finite field requirements of R1CS is the third challenge.

To overcome these challenges, we carefully designed the conversion process to convert UCB into a format compatible with zero-knowledge proofs. As shown in Figure.~\ref{zkUCB}, we eliminate randomness by inputting seeds into the pseudo-random number generator and solve the second and third challenges respectively through polynomial approximation and quantization operations. Details on how to solve the challenge can be found in Section.~\ref{overcome_challenge}. In zkUCB, we encapsulate all input, output, environmental, and the intermediate process of the UCB into one statement, and then use this statement to build a deterministic arithmetic circuit for subsequent execution of zk-SNARK.

\subsection{Contributions} In this article, we address key challenges in merging UCB with zk-SNARKs, focusing on data confidentiality and transparent decision process verification. We employ Groth16~\cite{groth2016size} for its efficient zk-SNARKs implementation, notable for its fast verification process and suitability for large blockchain networks. Additionally, Groth16's proofs are compact, facilitating easy transmission and verification, crucial for large-scale applications. Our main contributions include:

\begin{enumerate}
    \item \textbf{A new framework:} We introduce the Zero-Knowledge Proof into the MAB problem, a novel integration of zk-SNARKs (specify to Groth 16) with the UCB algorithm (an approach for addressing the MAB problem). 

    \item \textbf{Theory that links UCB and zero-knowledge proofs:} We convert the UCB algorithm into arithmetic circuits compatible with zero-knowledge proofs and overcome the challenges of randomness, non-polynomial operations, and floating point numbers in UCB. This transformation secures precise algorithm output validation and maintains data and parameter confidentiality, resulting in a transparent decision process verification.

    \item \textbf{A throughout evaluation:} Our experiments demonstrate that zkUCB outperforms the standard UCB in terms of reward. This means that our algorithm makes decisions that are closer to optimal than before the improvements. This improvement is because appropriate quantization bits can effectively reduce information entropy in the decision-making process, enhancing rewards. Particularly noteworthy is the fact that the proof size for zkUCB is proportional to the number of execution steps within the algorithm, suggesting that zkUCB maintains a manageable scale even when dealing with large datasets. This aspect is crucial in applications where balancing data security and operational efficiency is paramount.
    
\end{enumerate}

\begin{figure}
    \centering
    \includegraphics[width=0.45\textwidth]{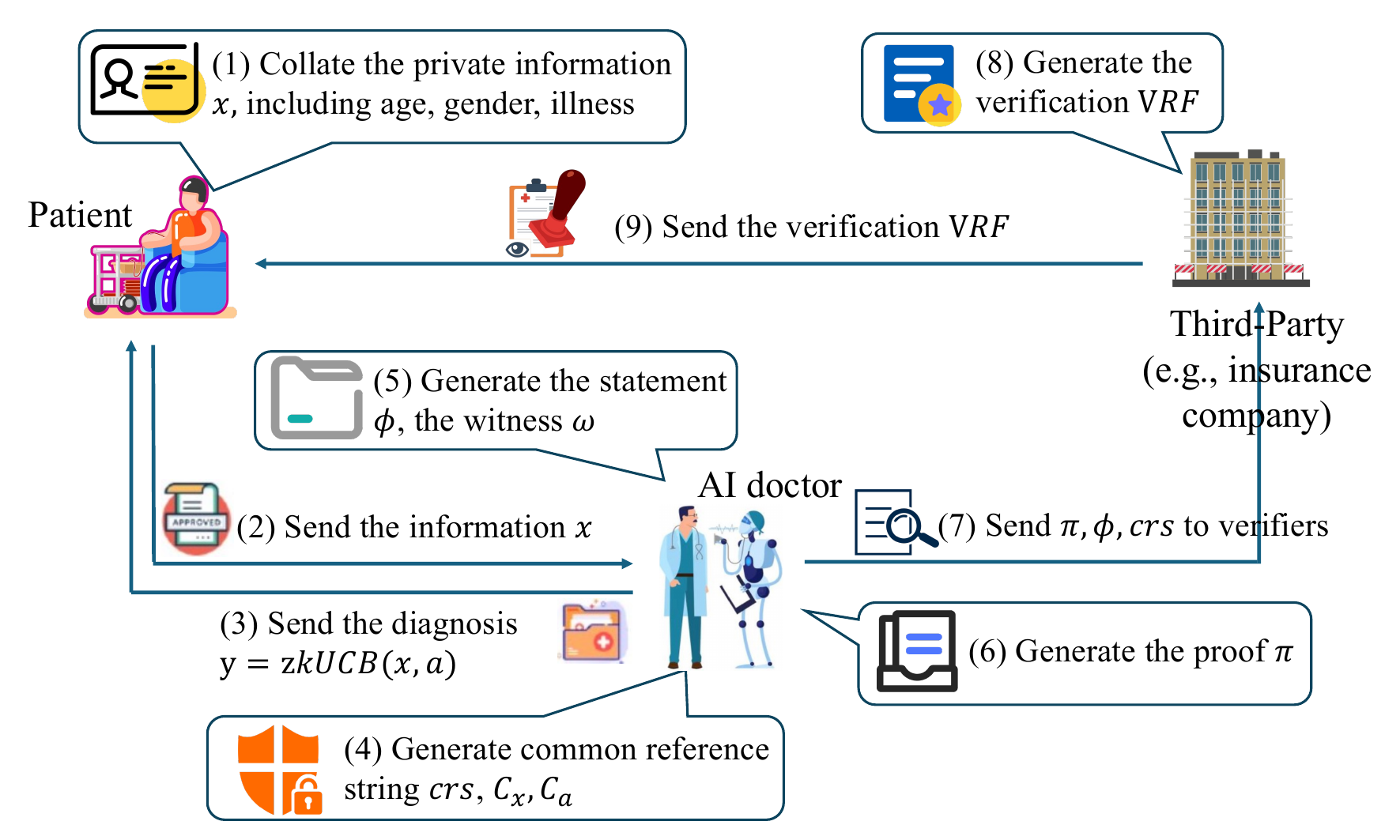}
    \caption{Application of zk-SNARKs in the AI Doctor System Using UCB Algorithm.}
    \label{AI doctor}
\end{figure}

\subsection{Use Case}\label{zkUCB_usecase}

Figure.~\ref{AI doctor} illustrates the application of zk-SNARKs to the AI doctor using the UCB algorithm. Here, the AI doctor functions as the prover, while the verifier could be either the patient themselves or a third party, such as an insurance company. Initially, the AI doctor gathers the patient's information, denoted as $x$, and computes a diagnostic result $y$ using the $zkUCB$ model and parameters $a$. Subsequently, the AI doctor creates a common reference string $crs$ and formulates commitments $C_x$ and $C_a$, based on $x$ and $a$ respectively. The next step involves the AI doctor generating a statement $\phi = (C_x, C_a, y)$ and a witness $w = (x, a)$, with the latter serving as a confidential input in the subsequent phase. The AI doctor then produces a proof $\pi = (crs, \phi, w)$ and dispatches $\pi$, the $crs$, and the statement $\phi$ to the third party for verification $VRF = (crs, \phi, \pi)$. After verification, the results $VRF$ are relayed to the patient. This ensures that commitments to the parameters and private data, along with the computational results and generated proofs, are presented to third parties for verification, all while safeguarding the integrity and confidentiality of both the data and algorithm.

\subsection{Organization}
The subsequent sections of this paper are structured as follows: In Section 2, we provide essential preliminaries. Section 3 delves into an in-depth discussion of zkUCB. Moving forward to Section 4, we present a thorough examination of comprehensive experimental results and analysis. In Section 5, we review the related work, highlighting the distinctions between our approach and existing methodologies. Finally, Section 6 encapsulates the paper with concluding remarks.

\section{Preliminaries}

\subsection{UCB Algorithm in MAB Problems}

The Multi-Armed Bandit (MAB) problem involves sequentially choosing from options with uncertain rewards to maximize cumulative rewards. The Upper Confidence Bound (UCB) algorithm addresses this challenge by balancing exploration (testing new options) and exploitation (leveraging known options). It achieves this equilibrium by constructing upper confidence bounds for each option, reflecting both anticipated rewards and uncertainty.

In this study, we focus on UCB1, a specific UCB algorithm version. The algorithm's workflow is outlined in Alg.~\ref{alg1}. It begins with an ``Initialization'' phase, playing each arm $i$ once to establish initial reward estimates ($x_i$). Following this, the algorithm updates the average reward $\bar{x}_i$ for each arm $i$ after it is played. Subsequently, the algorithm calculates a UCB value $\bar{x}_i + \sqrt{\frac{2 \cdot \ln n}{n_i}}$ for each arm $i$ and selects the arm with the highest UCB value. This strategic choice guides the algorithm toward achieving an optimal balance between exploring new arms and exploiting known ones.

\subsection{A General zk-SNARKs Scheme}
Zero-Knowledge Succinct Non-Interactive Argument of Knowledge (zk-SNARKs)~\cite{chen2022review} is a novel cryptographic protocol allowing a prover to authenticate a statement's truth to a verifier without extra information disclosure. zk-SNARKs stand out for their non-interactive nature, eliminating the need for continuous communication typical in traditional proofs. Their succinct proofs, small and fast to verify, remain efficient regardless of the complexity of the proven statement, enhancing their practicality in diverse applications.

\begin{algorithm}[t]
\caption{UCB1}
\label{alg1}
\begin{algorithmic}[1]
\STATE \textbf{Initialization:} Play each arm once
\STATE \textbf{Loop:} Play arm $j$ that maximizes $\bar{x}_j + 2 \sqrt{\frac{\ln n}{n_j}}$, where $\bar{x}_j$ is the average reward obtained from arm $j$, $n_j$ is the number of times arm $j$ has been played so far, and $n$ is the overall number of plays done so far.
\end{algorithmic}
\end{algorithm}

At first, the zk-SNARKs scheme involves converting computational statements (like the execution of a program) into a mathematical form (QAP) that can be efficiently proven and verified, as shown in Figure.~\ref{zkUCB}. Next, the workflow of zk-SNARKs mainly involves three stages: setup, proof, and verification, as follows:

\begin{enumerate}

    \item \textbf{Setup Phase:} zk-SNARKs scheme employs the setup algorithm to generate a common reference string ($crs$) including the evaluation key and verification key, which is used for proof generation and proof verification.

    \item \textbf{Proof Generation:} the prover employs the prove algorithm $Prove(\cdot)$ computes proof $\pi$ using the secret input $w$ (also named witness), the statement $\phi$, and the $crs$. I.e., $\pi \leftarrow Prove(crs, \phi, w)$, where the witness comprises private input essential for specific computations, such as input data and parameters for an ML model. It also encompasses all intermediate computational steps and results in converting high-level programs to QAP.

    \item \textbf{Proof Verification:} the verifier uses the verify algorithm $Verify(\cdot)$ to check the proof $\pi$ against the $crs$ to confirm the statement $\phi$'s truth without additional interaction. I.e., $0/1 \leftarrow Verify(crs, \phi, \pi)$, where ``1'' and ``0'' denote verification successful and failed, respectively.

\end{enumerate}

The zk-SNARK scheme can guarantee four properties, which are:
\begin{enumerate}
    \item \textbf{Completeness:} given a true statement $\phi$ and a proof $\pi$, both from the prover, an argument is complete if the probability of successful verification is $1$: \begin{equation}\label{Completeness}
    Pr[Verify(crs, \phi, \pi) = 1] = 1
    \end{equation}

    \item \textbf{Knowledge Soundness:} given a statement $\phi$ from a Probabilistic Polynomial-Time (PPT) adversary $\mathcal{A}$, who can forge the witness $w$ to generate a proof $\pi$. An argument is knowledge soundness if the probability that $\pi$ generated by $\mathcal{A}$ is successfully verified is almost $0$:
\begin{equation}\label{Knowledge Soundness}
    Pr[Verify(crs, \phi, \pi) = 1 | (\phi, \pi, w) \leftarrow \mathcal{A}] \approx 0
\end{equation}

    \item \textbf{Succinctness:} an argument is succinct if the size of the proof satisfies
\begin{equation}\label{Succinctness}
    |\pi| \le poly(k)poly \log (|\phi|+|w|)
\end{equation}
where $k$ refers to the security parameter, which dictates the system's security level and determines key sizes and cryptographic algorithm complexity.

    \item \textbf{Zero-knowledge:} given a simulation trapdoor $td$ for generating a simulated witness $w$, a simulator $Sim(\cdot)$ for creating a simulated proof that can pass the verification process without the true witness. Facing a  malicious verifier $\mathcal{A}$, an argument is zero-knowledge if $\mathcal{A}$ cannot discern whether the prover or the simulator generated the proof, i.e.,

\begin{equation}\label{Zero-knowledge}
\begin{split}
    &Pr[\mathcal{A}(crs,\phi,\pi) = 1 | \pi \leftarrow Prove(crs, \phi, w)]\\
    &= Pr[\mathcal{A}(crs,\phi,\pi) = 1 | \pi \leftarrow Sim(crs, \phi, td)]
\end{split}
\end{equation}

This indistinguishability denotes that the prover does not leak any information other than the truth of the statement, ensuring the essential security and privacy features of the zk-SNARKs scheme.

\end{enumerate}

From the above, zk-SNARK facilitates a secure, private, and efficient verification process in computing environments, making it invaluable in environments where data confidentiality and integrity are critical.

\section{zkUCB}

This section explores the integration of zk-SNARKs into the UCB algorithm. For efficiency, we utilize Zokrates, a toolkit providing both a circuit compiler and a proof system. This allows us to create a corresponding program code straightforwardly, which compiles the UCB algorithm into R1CS and uses Zokrates to generate CRS and zero-knowledge proofs. Next, we will first delve into how we tackled the implementation challenges of zkUCB. This will be followed by an overview of the zkUCB workflow, offering insights into its operational dynamics.

\subsection{Overcoming Challenges in Integrating zkUCB}\label{overcome_challenge}

In our research, we tackled the inherent challenges in implementing zkUCB by employing strategic solutions as follows.

\subsubsection{Pseudo-random Number Generator}

One of the fundamental challenges in integrating the UCB algorithm with zk-SNARKs arises from the need for randomness in MAB problems. Particularly, the UCB algorithm often requires random variables to decide between arms with identical maximum expected rewards. Moreover, in reinforcement learning scenarios, employing random variables with specific probability distributions is crucial for modeling uncertainty, optimizing policies, and facilitating adaptive decision-making in stochastic environments. However, zk-SNARKs demand determinism in the compiled programs, presenting a significant hurdle in blending reinforcement learning with zero-knowledge proofs. To overcome this, we incorporated a statistical pseudorandom number generator. This inclusion allows for the deterministic encoding of the UCB algorithm under the MAB model into a digital circuit, thereby transforming a fundamentally stochastic system into a deterministic one.

In the field of stochastic processes, the generation of discrete random variables conforming to a uniform distribution can be achieved by inputting a seed parameter into a linear congruential generator (LCG) denoted as $F$. The seed is a parameter determined by the user. In the execution of $F$, there is also a second input parameter $u$, which is the upper bound of its output. Consequently, the sequence of random numbers $\{X_1,...,X_n\}$ with length $n$ is generated by $X_i = F(seed, u), 1 \leq i \leq n$. Every $X_i$ belongs to the uniform distribution with mean $u/2$ and range of $(0, u)$.

\subsubsection{Non-polynomial Functions}

In the context of zk-SNARK, all programs must be performed by polynomial operations in finite fields. A critical aspect of zkUCB involves converting non-polynomial operations in the UCB algorithm, such as logarithms and non-integer powers, into polynomial forms. Our solution for handling logarithmic functions is the implementation of piecewise linear approximation by focusing on discrete intervals of positive integers. After performing the approximation, and rounding down results to the nearest integer, we achieved a practical and statistically robust approximation. Similarly, for approximating non-integer powers, particularly square roots as required in the UCB algorithm, we utilized Newton's method, fixed the number of iterations to $20$, and rounded down results to the nearest integer. These strategies allow us to approximate these functions using polynomial methods, ensuring compatibility with the computational constraints of zk-SNARKs while maintaining uniform time and space complexity.

\subsubsection{Quantization for Floating Point Numbers}

The UCB algorithm predominantly deals with real numbers, which inherently exist in a continuous space. Conversely, zk-SNARKs operate within finite fields, necessitating discrete computations. To bridge this gap, we employed a quantization process that maps the floating point numbers to positive integers. This process involves multiplying each element by a predetermined scaling factor ($q$). However, as depicted in Alg.~\ref{alg2}, directly applying proportional scaling to $n$ and $N_j$ is not viable. Such scaling significantly alters the value of $n$ away from zero, leading to a negligible derivative of $\ln{n}$, which in turn disrupts the crucial balance between exploration and exploitation. This imbalance is particularly evident in the learning agent's reduced inclination towards exploratory actions initially. Moreover, if proportional scaling is straightforwardly applied to $\ln{n}$ and $N_j$, the results remain essentially identical to those in the unscaled scenario. Yet, the scaled $\bar{x}_j$ becomes disproportionately larger than the other unscaled components, once again disturbing the exploration-exploitation equilibrium. By applying the scaling parameter uniformly to both the prior and subsequent components, we streamline the computations and attain more favorable results.

\begin{algorithm}[t]
\caption{zkUCB}
\label{alg2}
\begin{algorithmic}[1]
\STATE \textbf{Initialization:} Play each arm once
\STATE \textbf{Loop:} Play arm $j$ that maximizes $q \bar{x}_j + 2 q \left \lfloor{\sqrt{\frac{\left \lfloor{\ln n}\right \rfloor}{n_j}}}\right \rfloor$, where $\bar{x}_j$ represents the average reward obtained from arm $j$, $n_j$ is the number of how many times the arm $j$ has been played to date, $n$ denotes the overall number of plays done so far, and $\lfloor \cdot \rfloor$ indicates rounding down to the nearest integer.
\end{algorithmic}
\end{algorithm}

In our study, we assume that the reward from each arm follows a normal distribution. To make zkUCB satisfy this requirement, we use sliding window sampling to sample data from a uniform distribution to simulate a normal distribution based on the principle of the Central Limit Theorem. In zkUCB, we set the mean and the range to be $q$ and $(0, 2q)$, respectively for the uniform distribution. We repeat this sampling process for $20$ times and generate the average denoted as $(\bar{x})$. Obviously, $(\bar{x})$ belongs to a normal distribution with mean $q$ and range $(0, 2q)$. Since the mean of normal distribution satisfies the additivity,  this $(\bar{x})$ can be used to calculate the actual reward in zkUCB.


\subsection{The Workflow of zkUCB}

This subsection shows how the zkUCB scheme works.

\begin{itemize}
    \item \textbf{Setup} $(crs, td) \leftarrow Setup(zkUCB)$
    
     In this stage, the setup algorithm generates a common reference string $crs$ and a simulation trapdoor $td$. Where $crs$ includes the evaluation key and the verification key, $td$ is used to generate a simulated witness when checking the zero-knowledge properties of zkUCB to further generate simulated proofs $\pi$.
     
    \item \textbf{Proof Generation} $\pi \leftarrow Prove(crs, \phi, w)$

    In the proof generation stage, the prove algorithm computes proof $\pi$ using the secret input $w$ (also named witness), the statement $\phi$, and the $crs$. Where the statement $\phi$ is constructed using several elements, including the operations of a pseudo-random number generator determined by a seed, the expected rewards of each arm, the reward received after playing arm during each round, and the implementation of the UCB algorithm under the MAB model. The witness $w$ serves as private input, including the input data, parameters, all intermediate computational steps, and results in converting high-level programs to QAP.
        
    \item \textbf{Proof Verification} $0/1 \leftarrow Verify(crs, \phi, \pi)$
    
    In the proof verification phase, the verifier uses the verify algorithm to assess the proof $\pi$ against the $crs$, thereby ascertaining the truth of the statement $\phi$. Where ``1'' signifies successful verification, and ``0'' denotes failure.
\end{itemize}

\begin{theorem}
    The above zkUCB scheme satisfies a non-interactive zero-knowledge argument of knowledge with completeness and perfect zero-knowledge. It has computational knowledge soundness against a Probabilistic Polynomial-Time adversary.
\end{theorem}
\begin{proof}
    We prove this theorem by analyzing how the zkUCB scheme guarantees the properties mentioned in the theorem, as shown below.

\textbf{Completeness:} for the true statement $\phi$ and the proof $\pi$, both from the prover, the zkUCB is complete if the probability of successful verification is $1$: 
\begin{equation}\label{zkUCB_Completeness}
    Pr[Verify(crs, \phi, \pi) = 1] = 1
\end{equation}

\textbf{Knowledge Soundness:} given a statement $\phi^*$ from a Probabilistic Polynomial-Time (PPT) adversary $\mathcal{A}$, who can forge the witness $w^*$ to generate a proof $\pi^*$. the zkUCB has computational knowledge soundness if the probability that $\phi^*$ is successfully verified is close to $0$ and negligible:
\begin{equation}\label{zkUCB_Knowledge_Soundness}
    Pr[Verify(crs, \phi^*, \pi^*) = 1 | (\phi^*, \pi^*, w) \leftarrow \mathcal{A}] \approx 0
\end{equation}

\textbf{Zero-knowledge:} given a simulation trapdoor $td$ for generating a simulated witness $w'$, a simulator $Sim(\cdot)$ for creating a simulated proof $\pi'$ that can pass the verification process without the true witness $w$. Facing a  malicious verifier $\mathcal{A}$, the zkUCB is zero-knowledge if the malicious verifier $\mathcal{A}$ cannot discern whether the prover or the simulator generated the proof, i.e.,

\begin{equation}\label{zkUCB_Zero-knowledge}
\begin{split}
    &Pr[\mathcal{A}(crs,\phi,\pi) = 1 | \pi \leftarrow Prove(crs, \phi, w)]\\
    &= Pr[\mathcal{A}(crs,\phi,\pi') = 1 | \pi' \leftarrow Sim(crs, \phi, td)]
\end{split}
\end{equation}

This indistinguishability characteristic ensures that the prover reveals no information beyond the veracity of the statement. Consequently, neither a well-behaved verifier nor a malicious verifier can glean any additional knowledge from the statement $\phi$. This ensures the inherent security and privacy properties integral to the zkUCB scheme.
\end{proof}

In summary, zkUCB fulfills the criteria of completeness, perfect zero-knowledge, and computational knowledge soundness, against probabilistic polynomial-time adversaries. This makes zkUCB stand as a highly secure, private, and efficient framework, adept at facilitating transparent verification processes in computing environments. It is especially promising in scenarios where data confidentiality and integrity are of paramount importance.

\section{Experimental Evaluation and Analysis}

\subsection{Experimental Objectives}

In this section, we conduct a comprehensive assessment of the effectiveness of zkUCB. Our evaluation focuses on several key aspects:
\begin{enumerate}
    \item \textit{Reward Comparison:} We analyze the performance disparities between the standard UCB algorithm and zkUCB, with a particular emphasis on the average reward generated during each step across 100 iterations over 200 steps. 

    \item \textit{Time Analysis:} We delve into various time-related metrics for zkUCB, including setup time, compile time (the duration required for converting high-level programs to R1CS), compute witness time, generate proof time, and verify time. These measurements are conducted over 200 steps, with results recorded at intervals of 10 steps.

    \item \textit{Components Size:} We examine the sizes of essential components for zkUCB, including the proving key size, proof size, verifying key size, and witness size. These metrics are assessed over 200 steps, with measurements recorded at 10-step intervals.

\end{enumerate}

\subsection{Experimental Setup}

Our experiment leverages ZoKrates, as it enables the selection of a proof system, elliptic curves for pairing, and zk-SNARKs backend at our discretion. In this context, we have opted for the ALT\_BN128 elliptic curve. The Groth-16 protocol is chosen as the proof scheme. Additionally, the Ark backend is selected. The rationale for employing this specific configuration is attributed to its well-structured design, making it among the more favored options in current practical implementations. Specifically, we operate the following two settings for different experimental purposes. 

\subsubsection{Setting I} 
This experiment is under a three-armed MAB model and runs a UCB algorithm with 200 steps, where each arm has initially expected rewards set at 0.9, 1.0, and 1.1, respectively. We set three different levels of quantization ($2^4$, $2^8$, $2^{16}$) and calculate the average reward across $100$ iterations. To adapt to zk-SNARKs, we perform quantization operations on the UCB algorithm. To compare the expected reward of the quantized zkUCB with the non-quantified version of the standard UCB, we multiply the reward generated by the zkUCB by the reciprocal of the quantified scaling ratio, so that the reward generated by zkUCB and the reward generated by ordinary UCB belong to a same distribution, keeping the baseline consistent. The theoretical basis for this is the additivity of expectations derived from random variables.

\subsubsection{Setting II} 
Maintaining the same foundational parameters as in Setting I, this experiment executes a UCB algorithm with various steps (ranging from 20 to 200, with increments of 10 steps per interval) for $100$ iterations to average the results.

\subsection{Results and Analysis}

\subsubsection{Analysis of Reward}

Figure.~\ref{reward_analysis} compares the rewards generated by the standard UCB algorithm and zkUCB in each round. It demonstrates that zkUCB with smaller quantization bits slightly underperforms compared to the standard UCB. In contrast, zkUCB configured with larger quantization bits generally exceeds the performance of the standard UCB. Notably, the optimal performance is observed with zkUCB at a medium quantization level. This improvement is because appropriate quantization bits can effectively reduce information entropy in the decision-making process, enhancing rewards. However, overly high quantization may reduce information entropy excessively, causing a loss of vital information, which in turn leads to performance degradation of zkUCB. Consequently, a well-calibrated quantization level that optimally reduces information entropy is key to maximizing zkUCB's performance. 

\begin{figure}[t]
    \centering
    \includegraphics[width=0.4\textwidth]{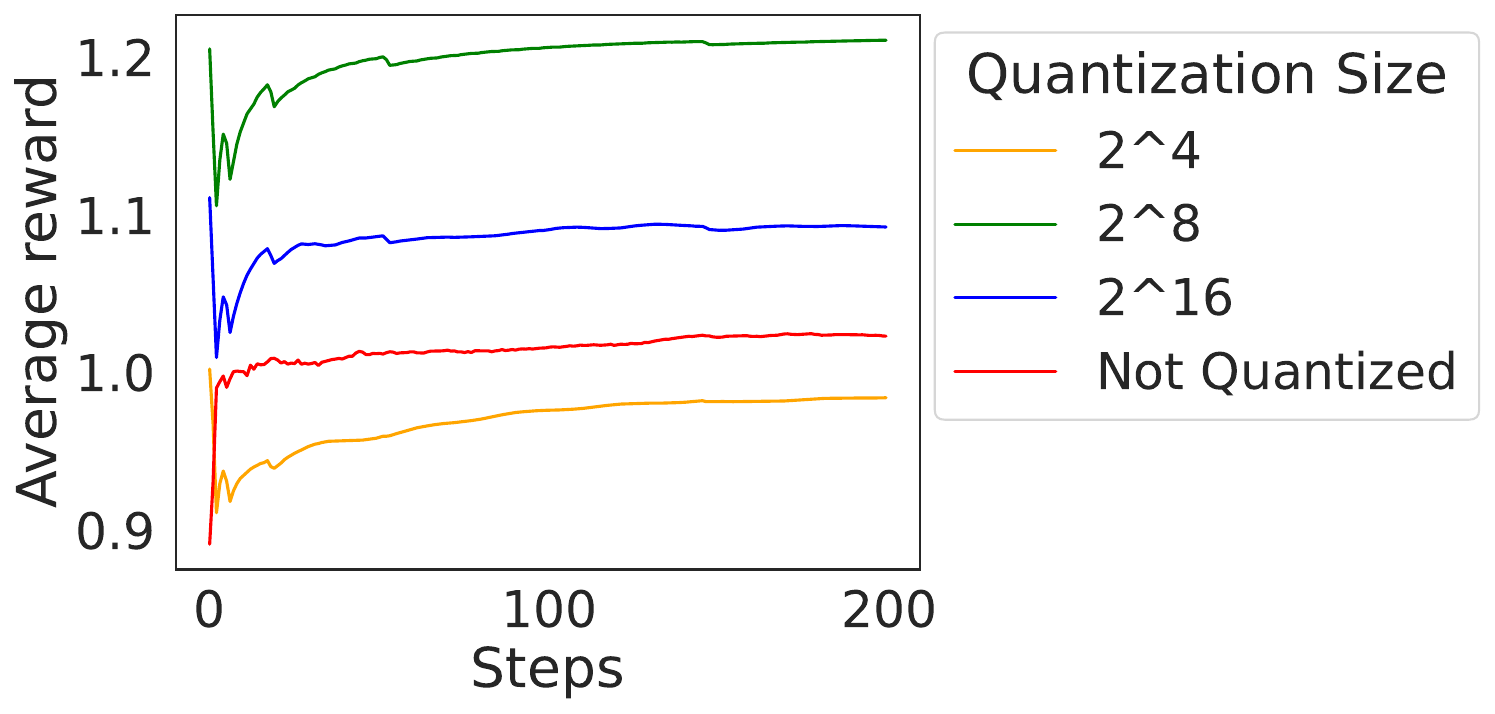}
    \caption{Average reward generated during each round by UCB and zkUCB over 100 iterations}
    \label{reward_analysis}
\end{figure}

\subsubsection{Time Efficiency Across Stages in zkUCB}

Figure.~\ref{time_analysis_1} and Figure.~\ref{time_analysis_2} display the average execution times for various stages of zkUCB with different quantization bits over a range of steps. Figure.~\ref{time_analysis_1} highlights consistent performance in setup, compilation, witness calculation, and proof generation times across various quantization sizes, showing these times are influenced more by the number of algorithm steps than by quantization bits. Figure.~\ref{time_analysis_2}, however, reveals a notable difference in verification times for zkUCB with different quantization bits; larger quantization numbers result in longer verification times, which also scale with the number of algorithm steps. These findings demonstrate zkUCB's promising application potential and scalability, especially considering the minimal verification times, which are within milliseconds.

\begin{figure}[t]
    \centering
    \begin{subfigure}[b]{0.24\textwidth}
        \includegraphics[width=\textwidth]{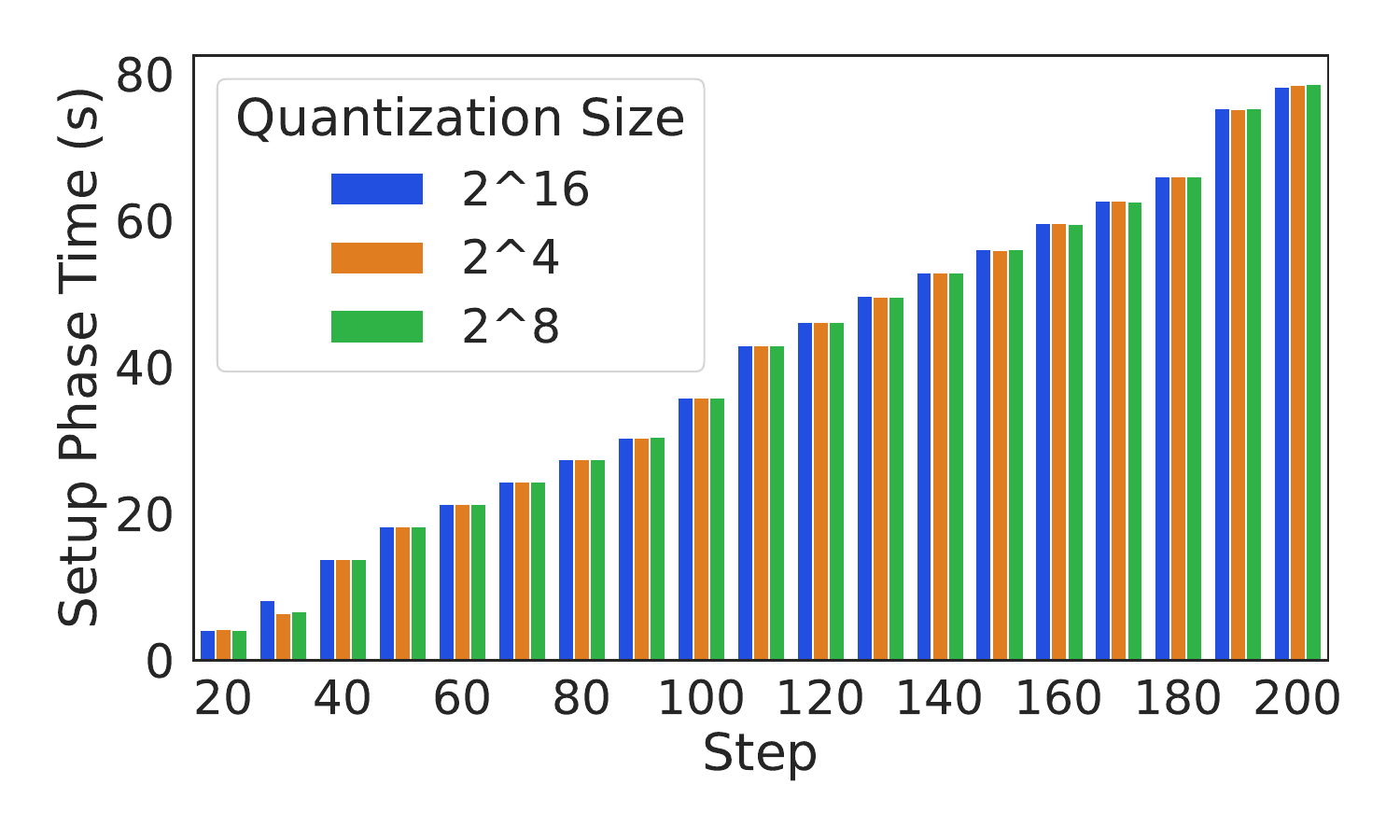}
        \caption{Setup phase time}
        \label{setup_time}
    \end{subfigure}
    \hspace{-0.1in}
    \begin{subfigure}[b]{0.24\textwidth}
        \includegraphics[width=\textwidth]{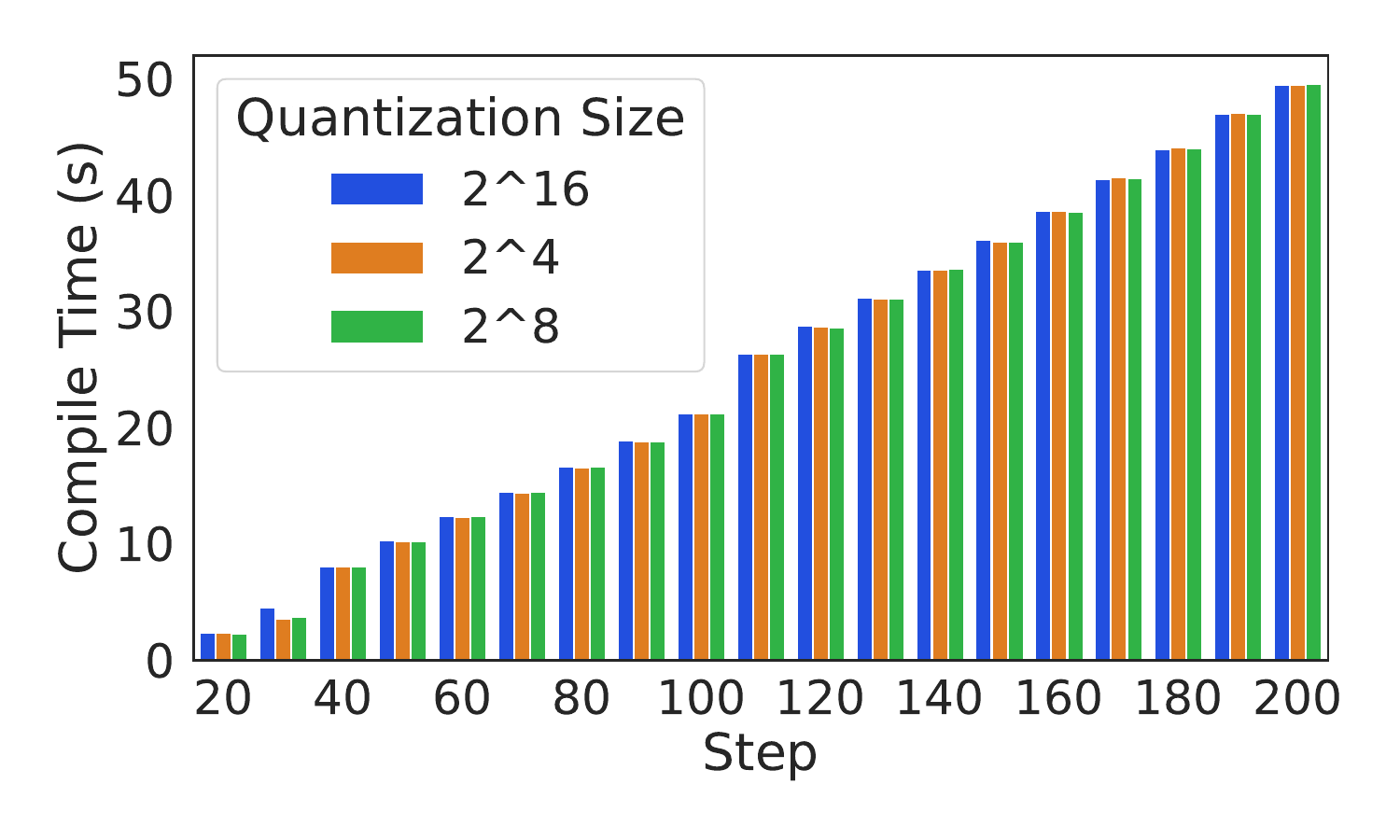}
        \caption{Compile time}
        \label{compile_time}
    \end{subfigure}
    \begin{subfigure}[b]{0.24\textwidth}
        \includegraphics[width=\textwidth]{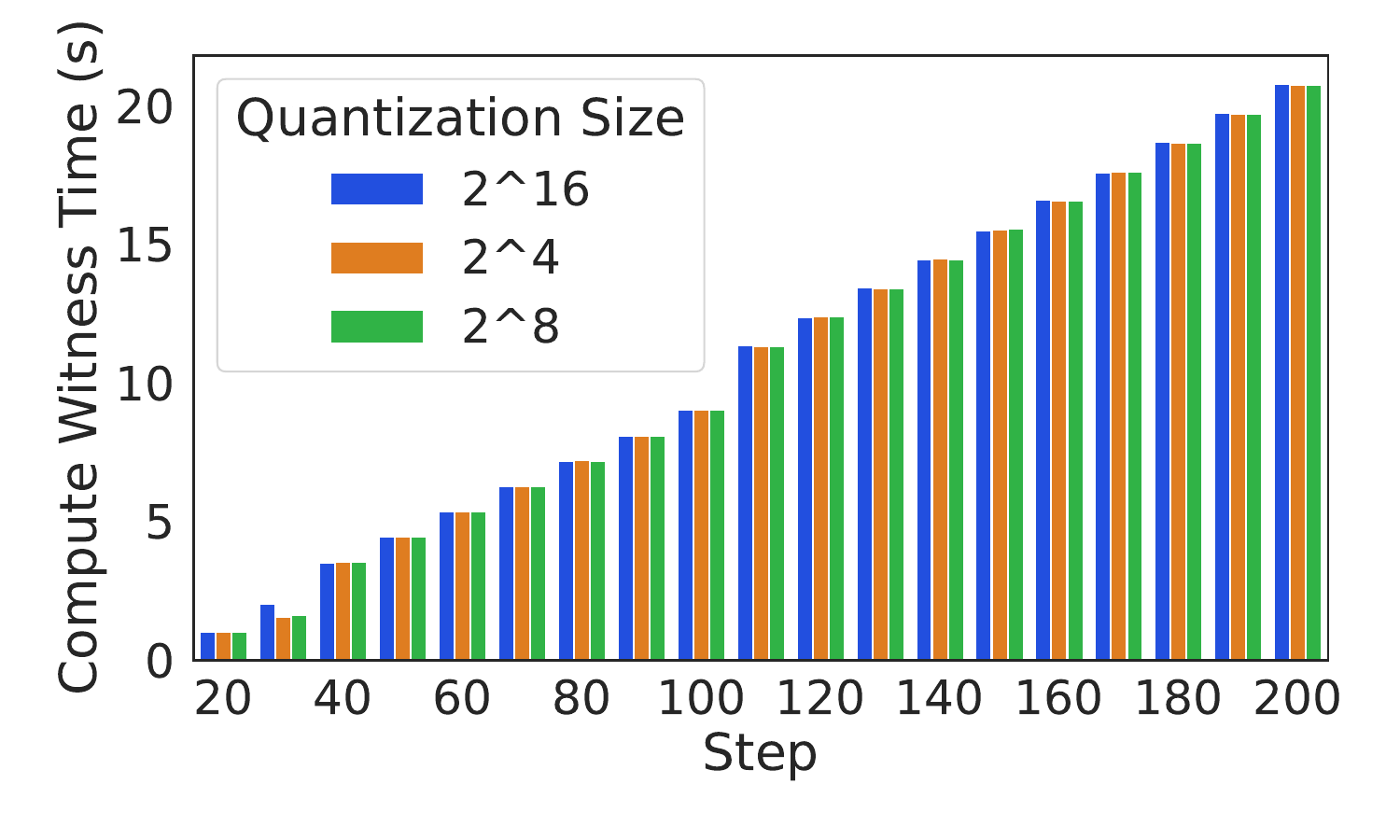}
        \caption{Compute witness time}
        \label{compute_witness_time}
    \end{subfigure}
        \hspace{-0.1in}
    \begin{subfigure}[b]{0.24\textwidth}
        \includegraphics[width=\textwidth]{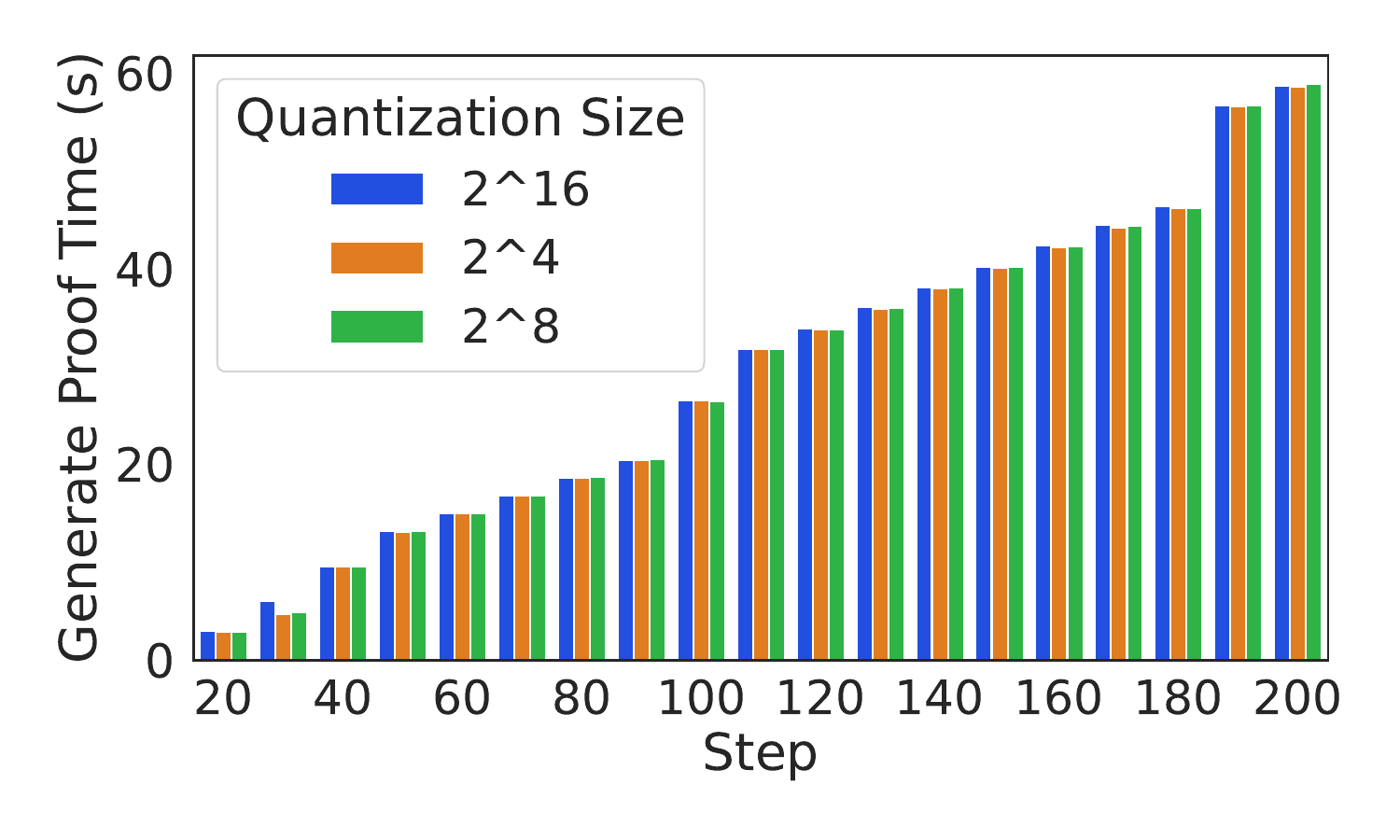}
        \caption{Generate proof time}
        \label{generate_proof_time}
    \end{subfigure}
    \caption{The average time during various phases}
    \label{time_analysis_1}
\end{figure}

\begin{figure}[t]
    \centering
    \begin{subfigure}[b]{0.24\textwidth}
        \includegraphics[width=\textwidth]{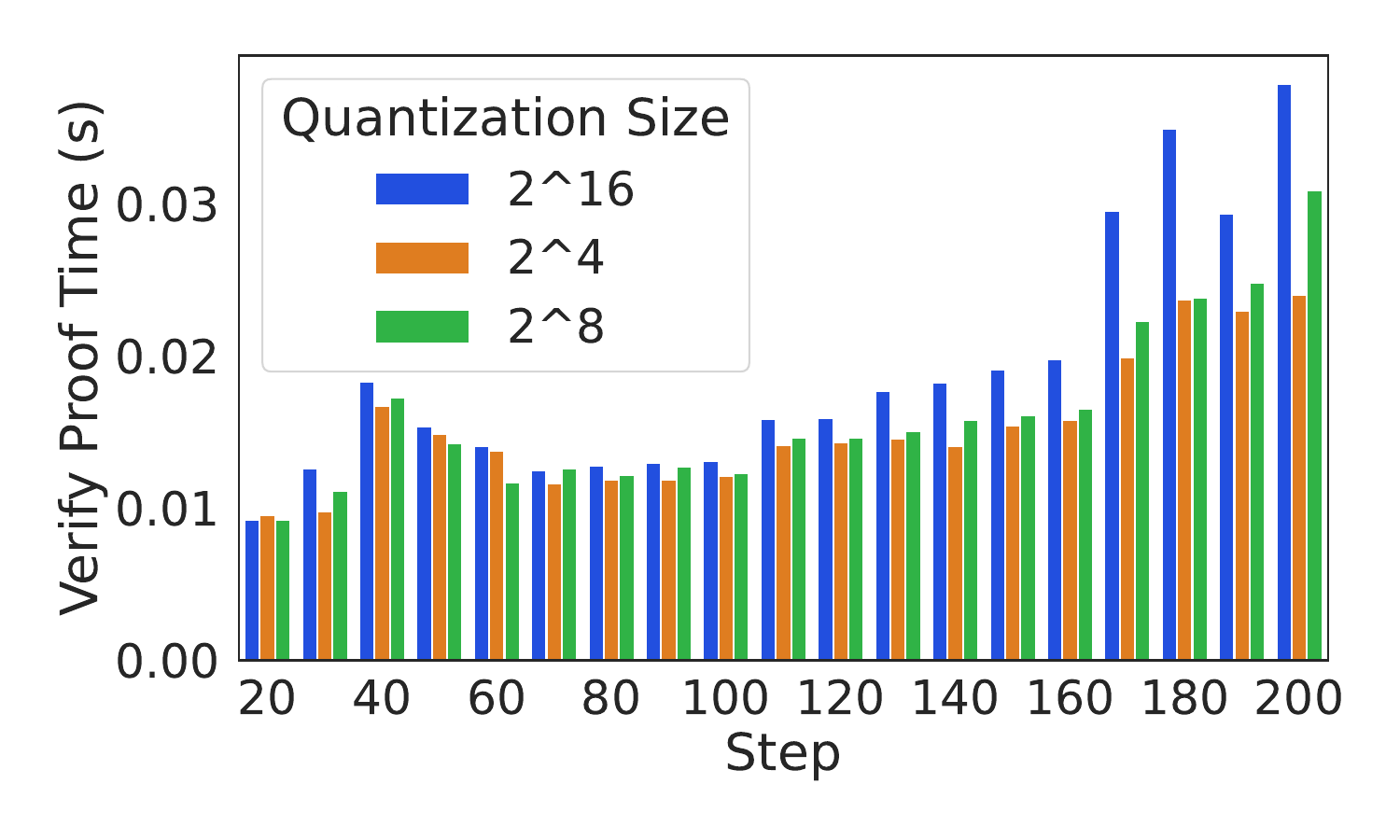}
        \caption{Histogram}
        \label{verify_proof_time_bar}
    \end{subfigure}
    \hspace{-0.1in}
    \begin{subfigure}[b]{0.24\textwidth}
        \includegraphics[width=\textwidth]{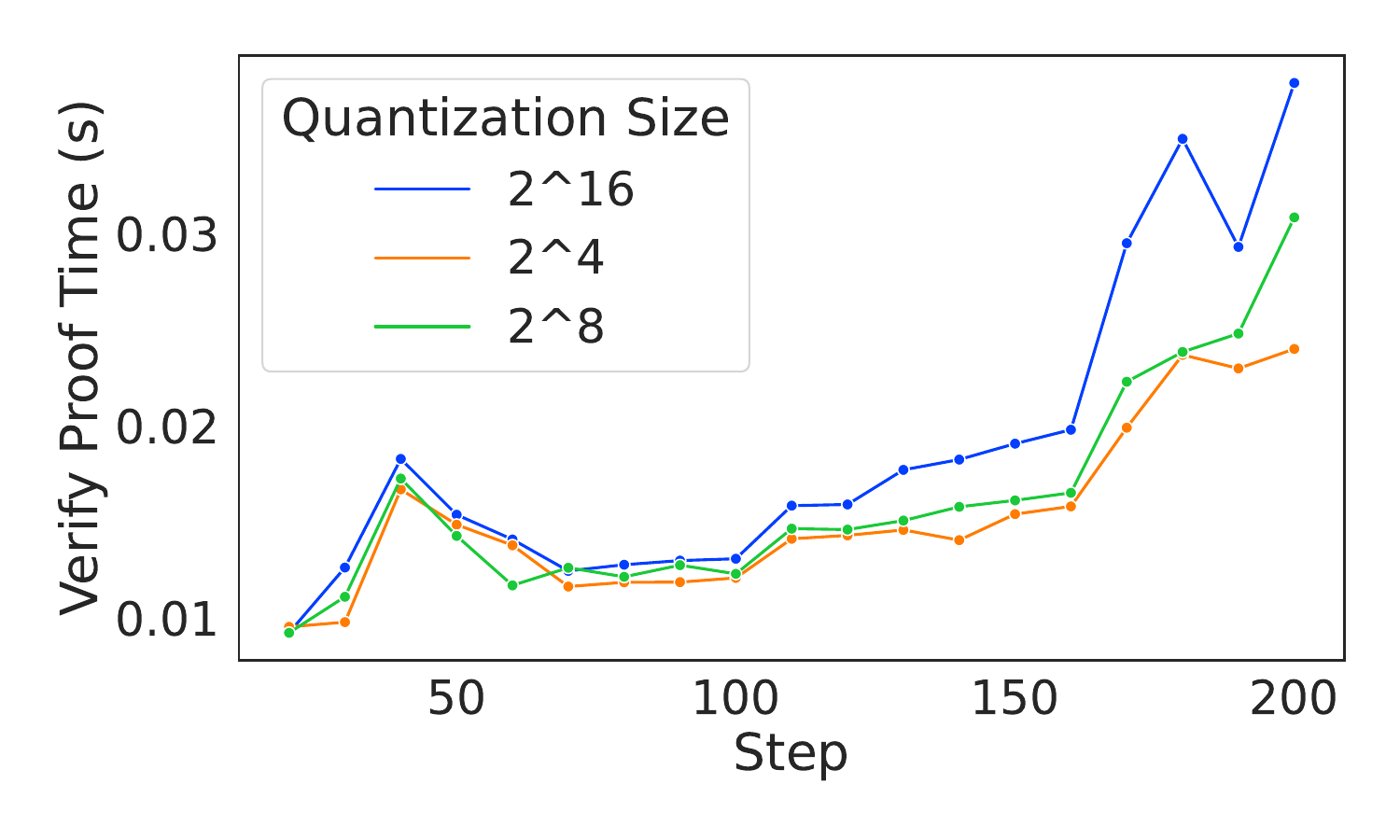}
        \caption{Line chart}
        \label{verify_proof_time_line}
    \end{subfigure}
    \caption{The average time for verifying proof}
    \label{time_analysis_2}
\end{figure}

\subsubsection{Components Size in zkUCB}

Figure.~\ref{components_size} displays the sizes of various components of zkUCB with different quantization bits across a sequence of steps, including witness size, proving (evaluation) key size, verifying (verification) key size, and proof size. The results reveal that the quantization bits have a negligible impact on the sizes of these zkUCB components. Instead, the size of these components is more influenced by the number of algorithm steps, showing a proportional relationship to these steps. Notably, even within 200 steps, the proof size remains relatively modest, only in the tens of kilobytes range, highlighting zkUCB’s considerable potential for broad application and scalability.

\begin{figure}[t]
    \centering
    \begin{subfigure}[b]{0.24\textwidth}
        \includegraphics[width=\textwidth]{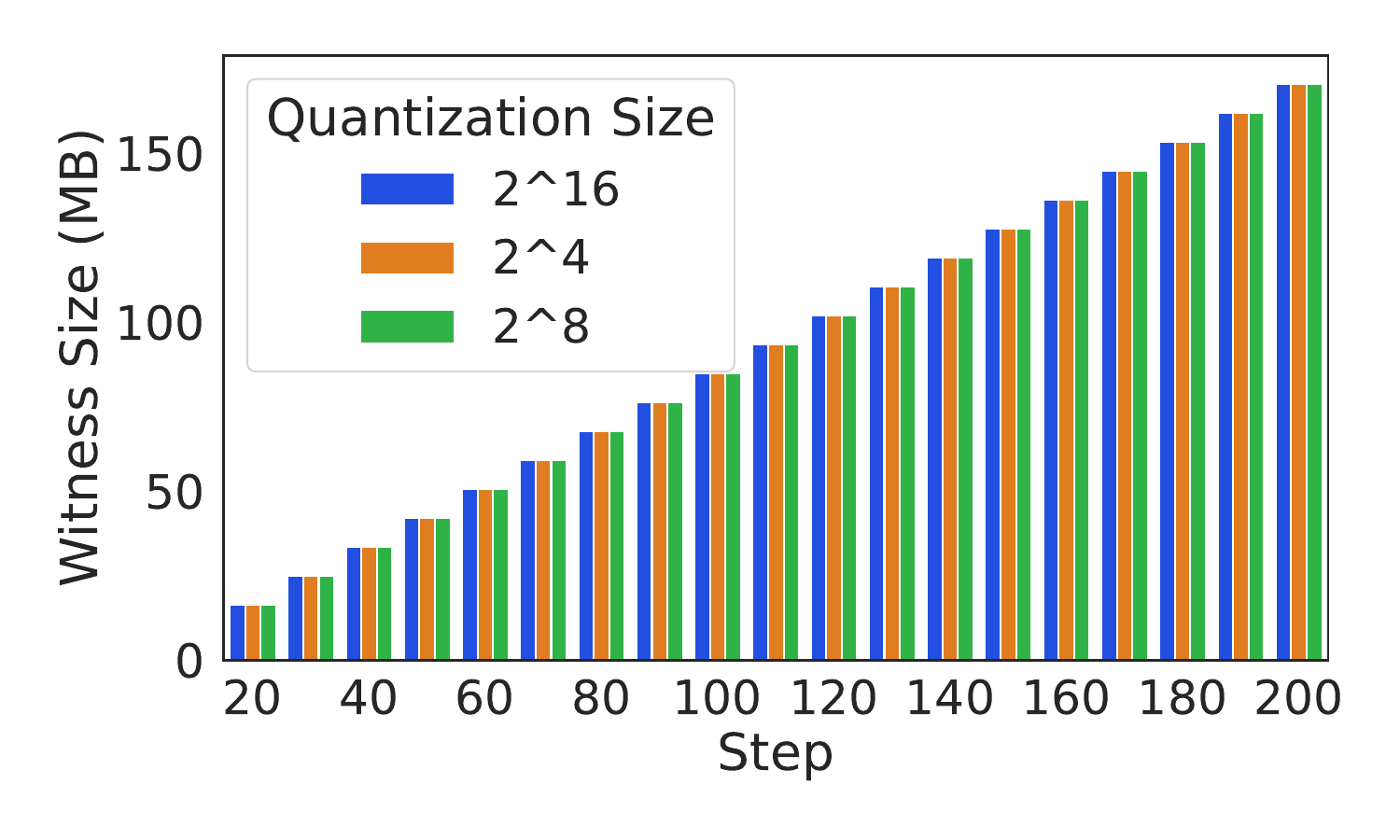}
        \caption{Witness size}
        \label{witness_size}
    \end{subfigure}
    \hspace{-0.1in}
    \begin{subfigure}[b]{0.24\textwidth}
        \includegraphics[width=\textwidth]{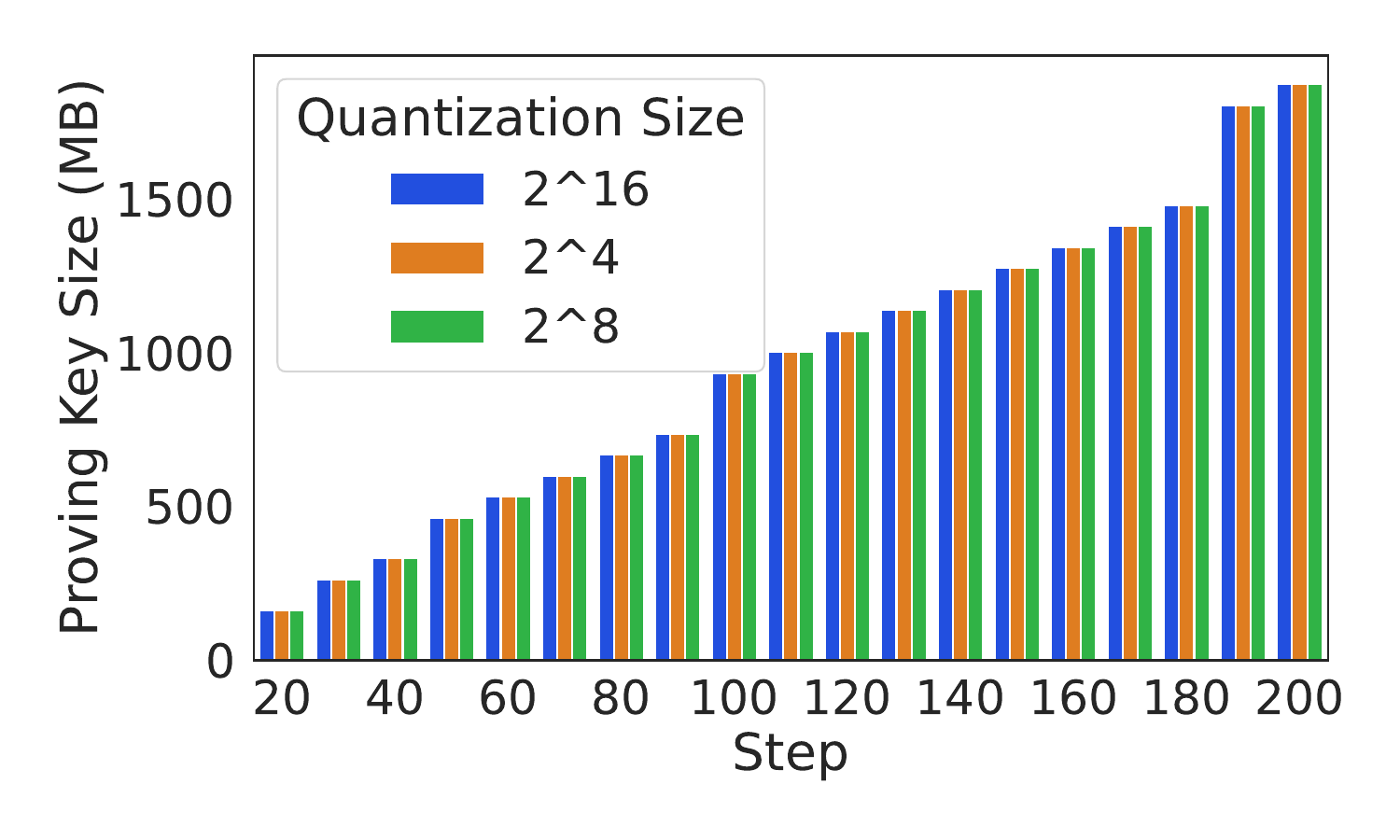}
        \caption{Proving key size}
        \label{proving_key_size}
    \end{subfigure}
    \begin{subfigure}[b]{0.24\textwidth}
        \includegraphics[width=\textwidth]{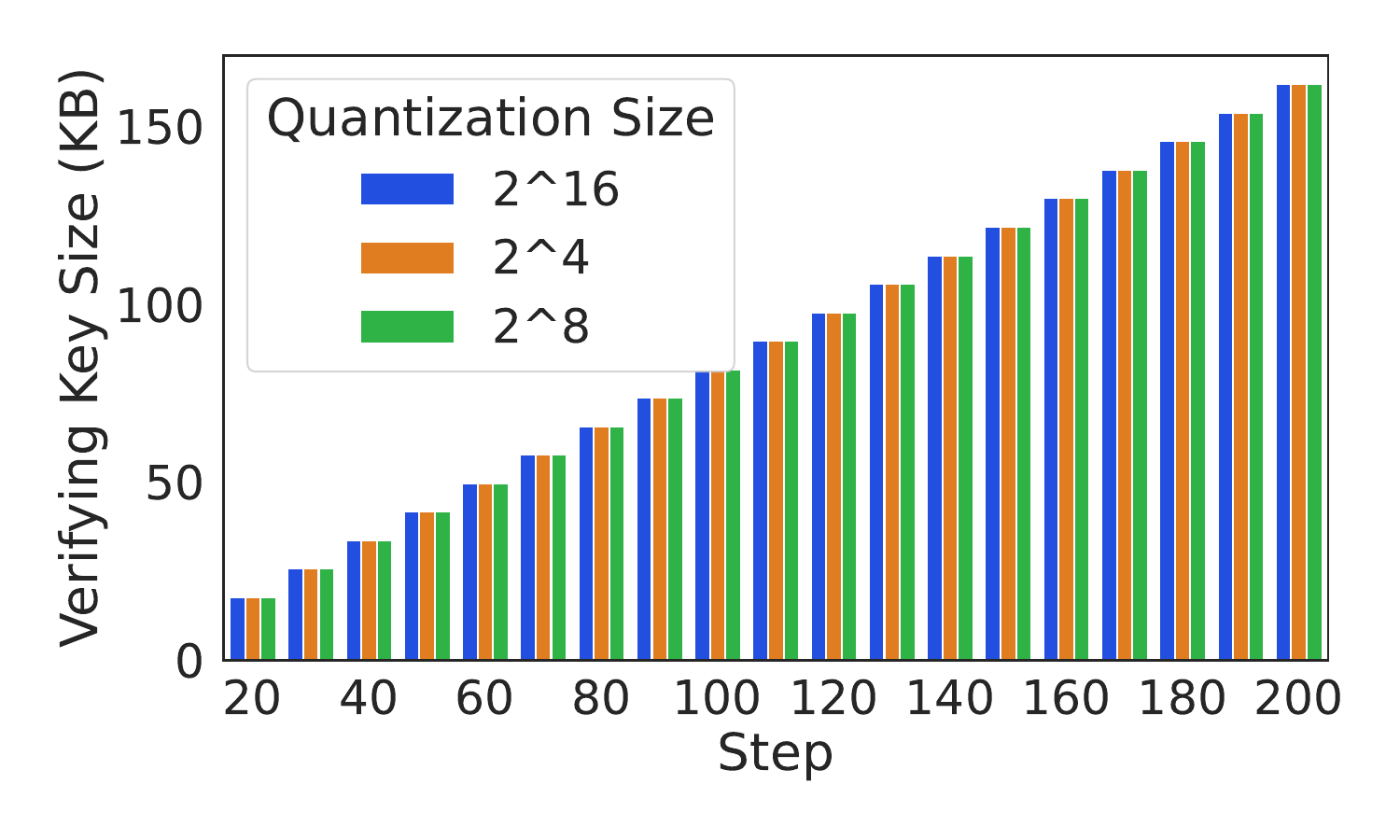}
        \caption{Verifying key size}
        \label{verifying_key_ize}
    \end{subfigure}
        \hspace{-0.1in}
    \begin{subfigure}[b]{0.24\textwidth}
        \includegraphics[width=\textwidth]{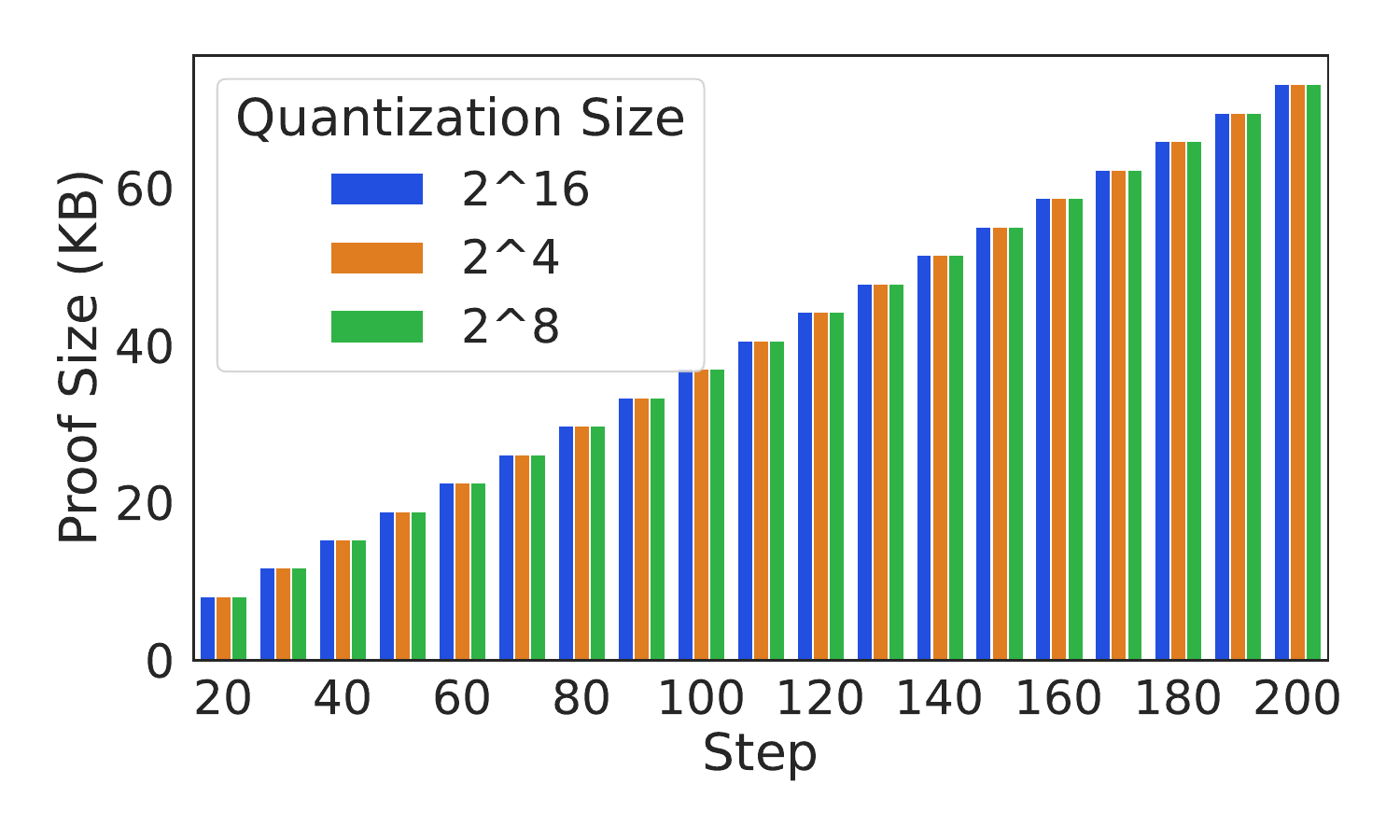}
        \caption{Proof size}
        \label{proof_size}
    \end{subfigure}
    \caption{The size of various components in zkUCB}
    \label{components_size}
\end{figure}

\section{Related Work} 
Here we first discuss the state-of-art for the verifiable computation then we move to discuss its implementation within machine learning. Finally, we discuss the shortcomings and how our approach could assist.

\paragraph{Verifiable Computation.} Verifiable computation is a key research area essential for validating the correctness of outsourced computations, particularly in cloud services. It focuses on verifying third-party computations without needing to re-execute them fully. This field encompasses a blend of cryptographic and non-cryptographic approaches to ensure computational integrity and data authenticity. Cryptographic methods include Zero-Knowledge Proofs~\cite{yang2022non,mouris2021zilch,wahby2018doubly}, Interactive Proofs (IP)~\cite{baelde2021interactive}, and Homomorphic Encryption (HE)~\cite{kim2023asymptotically}, which offer robust security guarantees. In contrast, non-cryptographic methods, notably Authenticated Data Structures (ADS)~\cite{sun2020adaptive}, concentrate on maintaining data integrity. Three primary streams dominate this field: ADS, which ensures data authenticity; IP, which facilitates interactive validation; and zk-SNARKs~\cite{groth2016size,wahby2018doubly}, renowned for efficient, non-interactive proof generation. These diverse methods highlight the field's complexity and richness, underpinning its significance in areas ranging from cloud computing to privacy-preserving protocols. This comprehensive approach to verifying computational processes without complete re-execution sets the foundation for advanced research and practical applications in computational integrity.

\paragraph{Verifiable Machine Learning.} The integration of verifiable computation with machine learning (ML) stands as a transformative advancement, poised to redefine data security and integrity in AI-driven applications. This convergence is not only enhancing the trustworthiness of ML models but is also unlocking new avenues for secure and private data analysis, particularly in sensitive sectors like healthcare, finance, and government. Traditional ML models often operate as black boxes, with limited ability to verify the internal processes and outputs. The advent of verifiable machine learning aims to bring transparency and reliability into this domain. Homomorphic encryption (HE)~\cite{li2020privacy} and garbled circuits (GC)~\cite{jawurek2013zero} are prevalent methods in this sphere, enabling secure computation of encrypted data. While these methods offer robust privacy guarantees, they often come with computational overheads. Another trend is the use of Zero-Knowledge Proofs in ML~\cite{weng2021mystique,xing2023zero}, which allows the validation of model predictions without revealing the underlying data or model parameters. This is particularly useful in scenarios where data confidentiality is paramount. Current research integrating ZKP has primarily concentrated on specific algorithms, including CNN~\cite{liu2021zkcnn,lee2024vcnn}, DNN~\cite{ghodsi2017safetynets}, linear regression, logistic regression, neural networks, support vector machines, K-means~\cite{zhao2021veriml}, and decision trees~\cite{zhang2020zero}.

\noindent\textbf{Shortcomings:} Despite the promising potential of ZKP for enhancing privacy, their application in Reinforcement Learning (RL) is underexplored. This is mainly due to RL's complexities and high computational demands in dynamic environments. We focus on the MAB problem, a prevalent model in critical domains like healthcare and recommendation systems, where maintaining data integrity is of paramount importance. We aim to apply ZKP's privacy-preserving capabilities to the UCB-based MAB framework, achieving a trustworthy and confidential decision process.

\section{Conclusion}
This study presents a significant advancement in the integration of reinforcement learning with data privacy through the development of the zkUCB algorithm. By integrating zk-SNARKs within the UCB algorithm, zkUCB ensures the confidentiality and security of sensitive data and parameters, while facilitating a transparent decision-making process. Our comprehensive experimental analysis confirms the superiority of zkUCB over the standard UCB algorithm, particularly in terms of reward optimization and operational efficiency. This improvement is because appropriate quantization bits can effectively reduce information entropy in the decision-making process, enhancing rewards. In addition, the scalability and practicality of zkUCB are demonstrated through its linear proof size and verification times, underlining its potential applicability in a wide range of real-world scenarios.

\section*{Acknowledgements}
This work is supported in part by the Shanghai Science and Technology Innovation Action Plan 23511100400, in part by the 2023-2024 Open Project of Key Laboratory Ministry of Industry and Information Technology-Blockchain Technology and Data Security 20242216, in part by the Data+ program: TRACK \& TRACE – Traceable Pharmaceutical Products project and PAPRiCaS: Programming technology foundations for Accountability, Privacy-by-design \& Robustness in Context-aware Systems by Independent Research Fund Denmark.


\bibliographystyle{named}
\bibliography{ijcai24}

\end{document}